\documentclass[conference]{IEEEtran}

\usepackage{graphicx}
\usepackage{algorithm,algpseudocode}
\usepackage{amsmath}
\usepackage{amsfonts}
\usepackage{booktabs}
\usepackage{hyperref}
\usepackage{xcolor}
\usepackage{tikz}
\usetikzlibrary{calc}
\usetikzlibrary{positioning}
\usepackage{comment}
\usepackage[table]{xcolor}
\usepackage{multirow}

\algnewcommand\algorithmicinput{\textbf{Input:}}
\algnewcommand\Input{\item[\algorithmicinput]}

\usepackage{amssymb,amsthm}

\newtheorem{proposition}{Proposition}

\IEEEoverridecommandlockouts

\title{\LARGE \bf 
CAST: Alternating State-Value Targets and Expanded Policy Gradients for Model-Based Reinforcement Learning}

\author{Pietro Noah Crestaz$^{1,2}$, Mohamed Yassine Kabouri$^{2,3}$, Nicolas Mansard$^{2,4}$, and Andrea Del Prete$^{1}$
\thanks{$^{1}$Industrial Engineering Department, University of Trento, Trento, Italy.}%
\thanks{$^{2}$LAAS-CNRS, Université de Toulouse, CNRS, Toulouse, France.}%
\thanks{$^{3}$Machines in Motion Laboratory, New York University, New York, USA.}%
\thanks{$^{4}$Artificial and Natural Intelligence Toulouse Institute (ANITI), Toulouse.}%
\thanks{$\star$Corresponding author: {\tt\small pietronoah.crestaz@unitn.it}. Project page: \url{https://pietronoah.github.io/cast/}}%

}

\begin{document}

\maketitle

\begin{abstract}
    Model-based reinforcement learning (MBRL) is a family of RL methods that learn a model of the environment and use it for action selection, making it well suited to robotics due to its sample efficiency. Combining learned models with online planning can further improve action selection, as the planner can exploit the model to find better actions than the learned policy alone. Recent methods combining learned policies with online planning typically learn the value of the policy rather than the stronger planner-guided behavior. We present CAST (Critic with Alternating State-value Target), which uses planner-guided behavior to improve value learning while regularizing the value estimate with the current policy. CAST replaces the action-value critic with a state-value critic, trained using a target that combines a real planner-guided transition and an imagined transition under the current policy. The resulting value function corresponds to an alternating process between planner-guided behavior and the current policy, allowing it to benefit from the stronger planner behavior while being regularised by the policy being learned. We evaluate CAST on the DeepMind Control and HumanoidBench Suites against several state-of-the-art methods, and demonstrate successful transfer to a physical Unitree Go2 quadruped performing a dynamic handstand.
\end{abstract}

\section{Introduction}
\label{sec:intro}

Reinforcement Learning (RL) has emerged as a powerful framework for learning complex robotic behaviors directly through interaction with the environment \cite{Sutton1998}. It has enabled learning skills such as dexterous manipulation \cite{openai_dexterous} and legged locomotion \cite{rudin2021learning}. However, applying RL in the real world remains challenging due to its high sample complexity. Model-based reinforcement learning (MBRL) addresses this limitation by learning a model of the environment and using it for planning and policy optimization, reducing the number of real-world interactions required. Planning-based MBRL performs online trajectory optimization on the learned model, evaluating and refining candidate action sequences before execution, providing a mechanism for improving action selection beyond the learned policy alone~\cite{schrittwieser2020muzero,silver2017alphazero}.

Recent methods combine online planning with an explicitly learned policy. MOPAC~\cite{mopac2021} learns both a $Q_{\pi_\theta}$ and a $V_{\pi_\theta}$, using MPPI rollouts to augment the training data for a model-free actor-critic policy. LOOP~\cite{sikchi2022loop} builds a planner on top of SAC~\cite{haarnoja2018sac}, while TD-MPC~\cite{hansen2022tdmpc}  and TD-MPC2~\cite{hansen2024tdmpc2} jointly learn a latent world model, a policy, and an action-value critic $Q(z,a)$ through temporal-difference (TD) learning, with the critic also providing the terminal value for planning. At run-time, the planner is used in a MPC setting and acts as a refinement $\beta$ of the learned policy $\pi_\theta$.

\begin{figure}[t]
\centering
\includegraphics[width=\columnwidth]{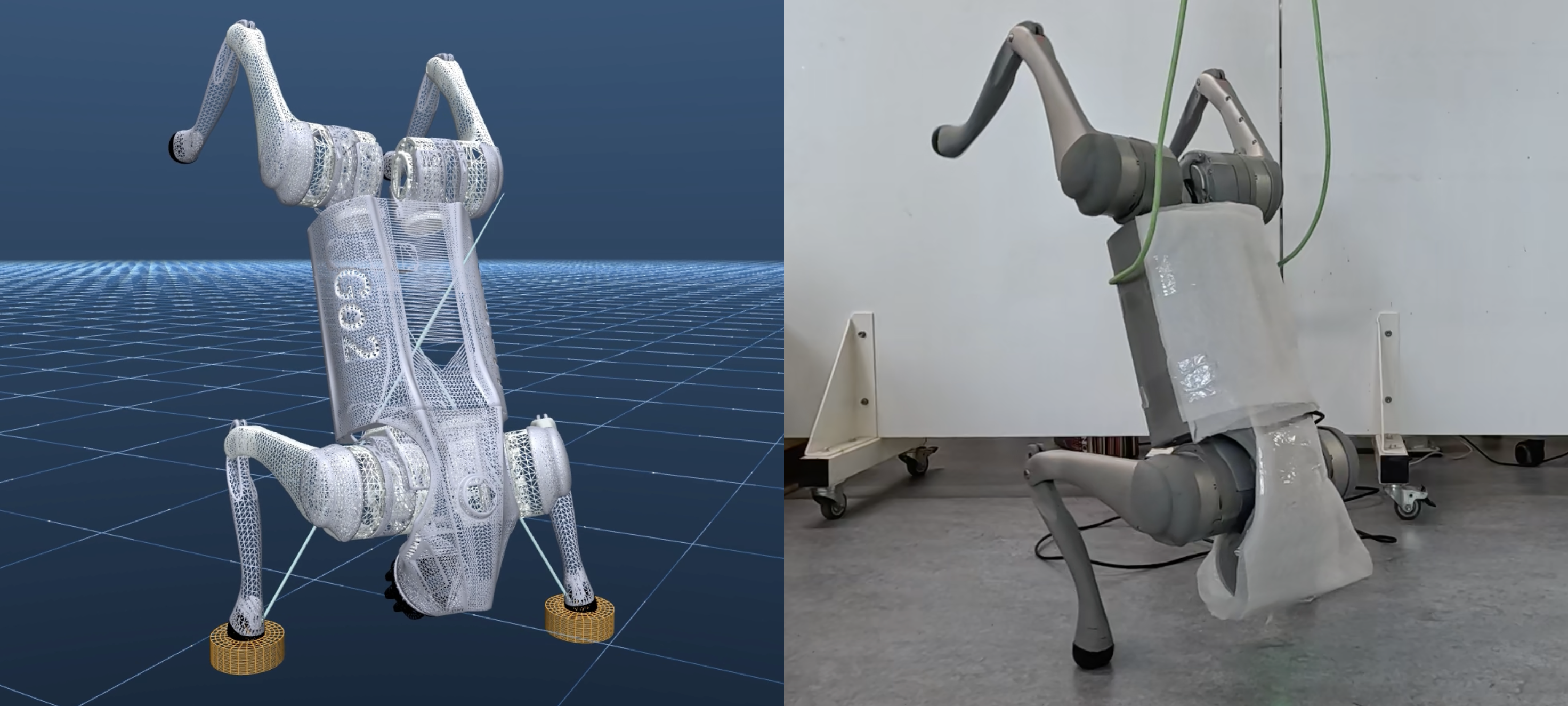}
\caption{\textbf{CAST sim-to-real transfer} on the Unitree Go2 quadruped performing a dynamic handstand. MuJoCo simulation (left) and real-robot execution (right).}
\label{fig:cast}
\end{figure}

Recent methods train the critic to learn the value of the optimized policy $\pi_\theta$. They differ mainly in how they close the gap between $\pi_\theta$ and the planner-augmented policy $\beta$ to address the well-known issue of distribution mismatch. BOOM~\cite{zhan2025boom} introduces a likelihood-free alignment loss that pulls the policy toward the planner's non-parametric action distribution. BMPC~\cite{wang2025bmpc} instead replaces critic-based policy improvement with behavior cloning (BC), using a state-value critic learned on-policy. Both methods therefore learn the value of $\pi_\theta$, and only differ in how they push $\pi_\theta$ closer to $\beta$.

We start instead from a simpler observation. The behavior policy $\beta$ (the planner) generally outperforms the learned policy $\pi_\theta$ \cite{hansen2024tdmpc2, wang2025bmpc}, especially in the early training stage. We argue that the critic should therefore learn the value of $\beta$ rather than the value of $\pi_\theta$, so that policy improvement can draw on a better critic. However, since $\beta$ generally outperforms $\pi_\theta$ by a wide margin, bootstrapping directly from $\beta$ for policy improvement yields large policy updates and destabilizes training, a common failure mode in RL methods. We address this with a regularized TD-target construction that mixes one real transition under $\beta$ with one imagined transition under $\pi_\theta$. Under fixed policies and exact dynamics, the resulting Bellman operator has a unique fixed point corresponding to an alternating $\beta \rightarrow \pi_\theta \rightarrow \beta \rightarrow \pi_\theta$ process.

This choice also raises a second question, whether this critic should be an action-value $Q_\beta(z,a)$ or a state-value $V_\beta(z)$. Learning $Q_\beta$ is impractical, since $\beta$'s action distribution at a state is only available by re-running the planner, requiring a new planner query for every TD target rather than just reusing stored transitions. A state-value critic $V_\beta(z)$ avoids this, depending only on the states $\beta$ visits, already recorded in the replay buffer. MODIP~\cite{modip2026} switches the learned policy for a diffusion policy, trained with BC over $\beta$, and trains the critic directly on planner-generated transitions, in order to avoid querying the policy during planning. The state-value formulation also changes policy optimization, since $V(z)$ takes no action as input and we cannot differentiate it directly with respect to the policy action as in the $Q(z,\pi_\theta(z))$ based objective. We instead reconstruct a one-step action-value estimate from the learned reward and dynamics models, and unroll $\pi_\theta$ through the model for $k$ steps before bootstrapping with $V$.

We build this formulation into TD-MPC2~\cite{hansen2024tdmpc2}, retaining its latent world model and MPPI planner, while replacing the action-value critic with an ensemble state-value critic. We call the resulting method \textbf{CAST} (\textbf{C}ritic with \textbf{A}lternating \textbf{S}tate-value \textbf{T}arget), a novel MBRL algorithm that replaces the action-value critic with a state-value critic trained toward the value of the behavior policy $\beta$. The key contributions are:

\begin{enumerate}

\item A regularized hybrid Bellman target with a characterized fixed point, combining a real transition under $\beta$ with an imagined transition under $\pi_\theta$;

\item A $k$-step model-based policy-improvement objective that unrolls $\pi_\theta$ through the learned dynamics before bootstrapping.

\end{enumerate}

CAST is validated first in simulation against five representative RL baselines on 14 high-dimensional control task, and finally on physical hardware by deploying it on a Unitree Go2 quadruped.

\section{Related Work}
\label{sec:related}

Planning with learned world models has become a central approach in MBRL. Sampling-based methods such as PlaNet~\cite{hafner2019planet} and TD-MPC/TD-MPC2~\cite{hansen2022tdmpc,hansen2024tdmpc2} use learned dynamics and reward models to generate predictive trajectories and bootstrap finite-horizon trajectories with learned value estimates. Other approaches, such as MuZero~\cite{schrittwieser2020muzero} and EfficientZero~\cite{wang2024efficientzerov2} combine learned latent dynamics with tree search. Among these methods, TD-MPC2 is most closely related to our setting. It combines latent world models, MPPI planning, and an action-value critic that is used both as the terminal reward for planning and as the policy improvement signal.

Our work is also related to recent methods that address the mismatch between the learned policy and the planner in planning-based RL. Replay data are generated by the planner-augmented behavior policy $\beta$, while gradient-based updates optimize the parametric policy $\pi_\theta$, introducing an inherent distribution mismatch. BOOM~\cite{zhan2025boom} addresses this mismatch by aligning the learned policy with the planner distribution using a forward KL divergence in the actor loss. The concurrent TD-MPC$^2$ policy-alignment method~\cite{lin2025tdmpc2policy} similarly regularizes the policy distribution toward planner actions, but uses a reverse KL divergence. BMPC~\cite{wang2025bmpc} instead uses a state-value critic trained using an on-policy model-based target and a Behavior Cloning (BC) loss for the actor. 

Our work also employs a state-value critic, but with a different goal. Rather than aligning $\pi_\theta$ with $\beta$ or imitating $\beta$'s actions, we aim to have the critic represent the value of $\beta$ directly, since $\beta$  generally outperforms $\pi_\theta$. Using $V_\beta$ to bootstrap the policy optimization, however, performs poorly in practice, as we show in Sec.~\ref{sec:ablations}. We therefore define a regularized hybrid Bellman target consisting of one replay transition under $\beta$ followed by one imagined transition under $\pi_\theta$, and characterize the value function to which this training procedure converges.

The policy optimization component of CAST is related to imagination-based RL. Dreamer~\cite{hafner2020dreamer} optimize the policy by differentiating through latent trajectories generated by learned dynamics, following the pathwise-gradient view of deterministic policy gradients~\cite{lillicrap2016ddpg}. DMO~\cite{amigo2025dmo} also propagates policy gradients through a learned dynamics model, while using a separate model for high-fidelity rollouts. These methods differ from TD-MPC2 in that online planning is not the central mechanism used during environment interaction. In TD-MPC2, the learned policy mainly provides an initial distribution for MPPI, while the planner produces the final actions executed in the environment. CAST combines this planning setup with model-based policy optimization by propagating gradients through multiple imagined transitions before bootstrapping with the state-value critic.

State-value functions have also been used in settings where explicit action-value estimation is undesirable or unnecessary. Implicit Q-Learning~\cite{kostrikov2022iql} learns a state-value function through expectile regression to approximate an upper expectile of $Q$, avoiding explicit maximization over out-of-distribution actions in offline RL, although it still learns an action-value function. POLO~\cite{lowrey2019polo} and other value-function-based MPC methods~\cite{crestaz2026tdcdmppi} use a learned state-value function as the terminal cost for model predictive control, which is closer to our use of $V(z)$ for planning. MODIP~\cite{modip2026} similarly replaces a policy-dependent action-value function with a terminal state value to avoid querying a diffusion policy during planning, reducing inference time. Our setting differs in that replacing $Q(z,a)$ with $V(z)$ is not just an architectural choice, but comes from our aim to learn the value of $\beta$ rather than $\pi_\theta$. 
Learning $Q$ of $\beta$ would indeed be extremely computationally demanding, requiring to explore both state and action spaces, while querying the planner $\beta$ for the construction of TD targets.
Learning $V$ is instead much cheaper, as it requires only the exploration of the state space, and we can still estimate $Q$ exploiting our knowledge of $V$ and the learned reward/dynamic model. CACTO~\cite{cacto, cacto-bic} similarly trains the critic on cost-to-go along TO solutions without querying the actor.

Finally, learned world models have also been deployed on physical robotic systems. MoDem-V2~\cite{lancaster2023modemv2} demonstrate TD-MPC-style world models for real-world manipulation, while subsequent work has explored adapting pretrained world models to hardware~\cite{feng2023finetuning}. RT-HCP~\cite{rt-hcp2025} addresses a closely related problem, compensating for inference delays that exceed the robot's control frequency.

\section{Method}

\begin{figure*}
    \centering
    \includegraphics[width=\textwidth]{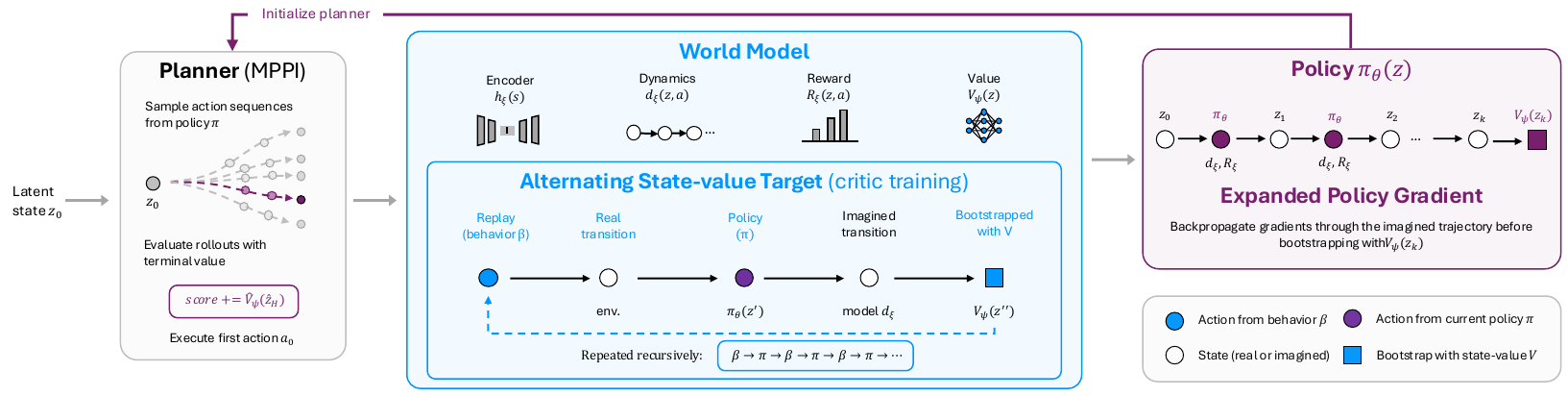}
    \caption{\textbf{Overview of CAST}. A world model learns latent dynamics, rewards, and a state-value critic trained with \textbf{alternating state-value targets} that combine a replay transition with an imagined policy transition. The planner evaluates candidate trajectories using the terminal state value, while the policy is optimized through \textbf{expanded policy gradients} propagated across imagined rollouts.}
\end{figure*}

\subsection{Background: TD-MPC2}
\label{sec:background}

We build on TD-MPC2~\cite{hansen2024tdmpc2}, which jointly learns a latent world model and a policy and interacts with the environment through online planning. Observations are mapped to a latent state $z$ using an encoder $z=h_\xi(s)$, while a latent dynamics head $z'=d_\xi(z,a)$ and reward head $r=R_\xi(z,a)$ predict the next latent state and reward without decoding back to observation space. An ensemble of $N{=}5$ action-value functions $\{Q^i_\psi(z,a)\}_{i=1}^N$ and a stochastic maximum-entropy Gaussian policy $\pi_\theta(a\mid z)$ complete the model. All components are MLPs with LayerNorm and Mish activations. The latent state is projected onto a product of simplices using SimNorm~\cite{lavoie2022simnorm} to keep its scale bounded and sparse. Reward and value targets are represented as soft cross-entropy distributions over symlog-spaced bins~\cite{bellemare2017c51} rather than raw scalars, making training less sensitive to the scale of the task reward. A target ensemble $\bar Q_\psi$ is maintained using Polyak averaging,
$\bar\psi \gets \tau\psi + (1-\tau)\bar\psi$,
and TD targets use the minimum of two randomly sampled ensemble members.
Given a $H$-step trajectory $(s_t,a_t,r_t,s_{t+1})_{t=0}^{H}$ sampled from the replay buffer, the encoder, dynamics, reward, and value heads are trained jointly by minimizing
\begin{equation}
\begin{split}
\mathcal L_{\text{model}} = &\mathbb E\Big[\textstyle\sum_{t=0}^{H}\rho^t\big(|d_\xi(z_t,a_t)-\mathrm{sg}(h_\xi(s_{t+1}))|2^2 \\
&+ \mathrm{CE}(R_\xi(z_t,a_t),r_t) + \mathrm{CE}(Q_\psi(z_t,a_t),y_t)\big)\Big],
\end{split}
\label{eq}
\end{equation}
where $\mathrm{sg}$ denotes the stop-gradient operator, $\rho\in(0,1]$ discounts errors further along the rollout, and $\mathrm{CE}$ is cross-entropy against a two-hot target. The one-step TD target for the action-value function is
\begin{equation}
y_t = r_t + \gamma,\bar Q_\psi\big(z_{t+1},\tilde a_{t+1}\big),
\qquad
\tilde a_{t+1}\sim\pi_\theta(\cdot\mid z_{t+1}).
\label{eq}
\end{equation}
This implies that TD-MPC2 learns the value of the neural policy $\pi_\theta$, as this is the policy used to sample $a_{t+1}$. However, the transitions in the replay buffer are not generated by $\pi_\theta$.
At each control step, actions are selected using the Model Predictive Path Integral (MPPI) planner \cite{williams2017mppi}. The planner samples horizon-$H$ action sequences from a diagonal Gaussian $\mathcal N(\mu,\sigma^2)$, rolls them out through $d_\xi$, and scores each sequence using
\begin{equation}
\hat G(z_t, a_{t:t+H}) =
\sum_{h=0}^{H-1}\gamma^h R_\xi(z_{t+h},a_{t+h})
+
\gamma^H \bar Q_\psi\big(z_{t+H}, a_{t+H}\big),
\label{eq}
\end{equation}
where the terminal action $a_{t+H}\sim\pi_\theta(\cdot\mid z_{t+H})$ is sampled from the policy. A fraction of the sampled sequences are initialized from the policy prior, so $\pi_\theta$ contributes a subset of the candidate rollouts at each planning step. An elite set is selected from the rollout batch, and the mean of the sampling distribution is updated using a softmax over the elite scores. This procedure is repeated for a fixed number of iterations, after which the planner executes the first action of the resulting sequence. The policy prior is trained separately using a maximum-entropy objective similar to standard Soft Actor-Critic~\cite{haarnoja2018sac},
\begin{equation}
\begin{split}
\mathcal L_\pi(\theta) = -\mathbb E\Big[\textstyle\sum_{t=0}^{H} \rho^t \big(&\alpha\,\mathcal H(\pi_\theta(\cdot\mid z_t)) \\
&+ Q_\psi(z_t,\tilde a_t)\big)\Big],
\end{split}
\label{eq}
\end{equation}
where $\tilde a_t\sim\pi_\theta(\cdot\mid z_t)$, $\alpha$ is a fixed entropy coefficient, and gradients are taken only with respect to $\theta$.

\subsection{From Action-Values to State-Values}
\label{sec:qtov}

Our method replaces the action-value critic with a state-value critic. For any stochastic policy $\pi$, the corresponding action-value and state-value functions satisfy
\begin{align}
Q_\pi(z,a)
&=
R(z,a)
+
\gamma V_\pi\big(d(z,a)\big),
\label{eq:bellman_q}\\
V_\pi(z)
&=
\mathbb E_{a\sim\pi(\cdot\mid z)}
\big[Q_\pi(z,a)\big].
\label{eq:bellman_v}
\end{align}
Here, the transition dynamics is deterministic, while the expectation in~\eqref{eq:bellman_v} is taken over the stochastic policy.

Following the value averaging of TDMPC2 \cite{hansen2024tdmpc2}, we similarly learn an ensemble of state-value functions $\{V^i_\psi(z)\}_{i=1}^N$ with the same architecture as the original TD-MPC2 critic. A target ensemble $\{\bar V_\psi^i\}_{i=1}^N$ is maintained using the same Polyak averaging rate $\tau$. Whenever an action-value estimate is needed, we reconstruct a one-step estimate from the learned reward, dynamics models and  target state-value function,
\begin{equation}
\hat Q(z,a)=R_\xi(z,a)+\gamma\bar V_\psi\big(d_\xi(z,a)\big).
\label{eq:composed_q}
\end{equation}
This composed estimate is the starting point to define the value-learning target and to construct the policy objective in Sec.~\ref{sec:expanded}.

Using a state-value critic also changes the terminal objective used by the planner. Rather than sampling a terminal action and evaluating $\bar Q_\psi(z_{t+H},a_{t+H})$, CAST uses the state-value estimate directly,
\begin{equation}
\hat G(z_t, a_{t:t+H})= \sum_{h=0}^{H-1}\gamma^h R_\xi(z_{t+h},a_{t+h})+\gamma^H\bar V_\psi\big(z_{t+H}\big).
\label{eq:vpi_value}
\end{equation}
This removes the need to sample a terminal policy action during planning.

\subsection{A Hybrid Value Target: Mixing Behavior and Policy}
The first design choice in CAST is how to train $V_\psi$. We aim for the critic to represent the value of $\beta$ rather than of $\pi_\theta$, but bootstrapping directly from the value of $\beta$ performs poorly in practice because it leads to abrupt policy updates that destabilize training (see ablations in Sec.~\ref{sec:ablations}). We therefore combine one real transition generated by $\beta$ with one imagined transition generated by the learned policy $\pi_\theta$.

Starting from a replay transition $(s_t,r_t,s_{t+1})$, we train the critic against the following two-step target using the same soft cross-entropy formulation used by TD-MPC2:
\begin{equation}
y_t=r_t+\gamma\Big[ R_\xi(z_{t+1},\tilde a_{t+1})+\gamma\,\bar V_\psi(\tilde z_{t+2})\Big],
\label{eq:y_2step}
\end{equation}
where $z_{t+1} = h_\xi(s_{t+1})$, $\tilde a_{t+1} \sim \pi_\theta(\cdot\mid z_{t+1})$, and $\tilde z_{t+2}= d_\xi(z_{t+1},\tilde a_{t+1})$.
The first step is the real transition stored in the replay buffer, while the second step is generated under the current policy using the learned model. 
The value function therefore corresponds to the alternating policy sequence
\begin{equation}
\beta\rightarrow\pi_\theta\rightarrow\beta\rightarrow\pi_\theta\rightarrow\cdots.
\label{eq:alternating_policy_sequence}
\end{equation}
The critic is neither the value function of $\beta$ nor that of $\pi_\theta$ alone. Instead, it is the value function induced by their alternating composition.

This construction regularizes the critic toward $V_\beta$ without collapsing onto it entirely. The replay transition keeps the target anchored to the stronger behavior policy $\beta$, while the imagined transition reintroduces $\pi_\theta$ every other step, preventing the large, destabilizing policy updates that a critic trained purely on $\beta$ would otherwise produce.

\subsection{Theoretical Analysis of Hybrid Value Targets}
We next characterize formally the value function represented by this target. Assume that $\beta$ and $\pi_\theta$ are fixed, the environment and learned dynamics are deterministic, and the value function can be represented without approximation error. Let $\beta(a\mid z)$ denote the stochastic behavior policy, which is stochastic due to the stochastic nature of the planner, and $\pi_\theta(a\mid z)$ the stochastic learned policy. The corresponding one-step Bellman operators are
\begin{align}
(\mathcal T_\beta f)(z)&=\mathbb E_{a\sim\beta(\cdot\mid z)}\left[R(z,a)+\gamma f\big(d(z,a)\big)\right],
\label{eq:t_beta}
\\
(\mathcal T_\pi f)(z)&=\mathbb E_{a\sim\pi_\theta(\cdot\mid z)}\left[R(z,a)+\gamma f\big(d(z,a)\big)\right].
\label{eq:t_pi}
\end{align}
Although the dynamics are deterministic, both operators remain stochastic through their respective action distributions.

The target in~\eqref{eq:y_2step} corresponds to the composition
\begin{equation}
\mathcal T_{\mathrm{CAST}}=\mathcal T_\beta\mathcal T_\pi.
\label{eq:operator_composition}
\end{equation}
\begin{proposition}
\label{prop:cast_target}
Assume that $\beta$ and $\pi_\theta$ are fixed, the environment and learned dynamics are deterministic, and the value function can be represented without approximation error. Then $\mathcal T_{\mathrm{CAST}}$ is a $\gamma^2$-contraction under the supremum norm and admits a unique fixed point $V_{\text{CAST}}$ satisfying
\begin{equation}
V_{\text{CAST}}=\mathcal T_\beta\mathcal T_\pi V_{\text{CAST}}.
\label{eq:hybrid_fixed_point}
\end{equation}
\end{proposition}

\begin{proof}
Expanding the composition gives
\begin{equation}
\begin{aligned}
&(\mathcal T_\beta\mathcal T_\pi V)(z_0) = \mathbb E_{a_0\sim\beta(\cdot\mid z)}\Big[R(z_0,a_0)\\
&+\gamma\,\mathbb E_{a_1\sim\pi_\theta(\cdot\mid z_1)}\big[R(z_1,a_1)+\gamma V(z_2)\big]\Big],
\end{aligned}
\label{eq:hybrid_operator}
\end{equation}
where $z_1=d(z_0,a_0)$ and $z_2=d(z_1,a_1)$. Thus, one application of the CAST operator consists of one transition under the behavior policy followed by one transition under the current policy.

Both $\mathcal T_\beta$ and $\mathcal T_\pi$ are individually $\gamma$-contractions under the supremum norm \cite{Sutton1998}. Their composition is therefore a $\gamma^2$-contraction:
\begin{equation}
\begin{aligned}
\left\|\mathcal T_\beta \mathcal T_\pi V
      - \mathcal T_\beta \mathcal T_\pi U\right\|_\infty
&\leq
\gamma\left\|\mathcal T_\pi V
      - \mathcal T_\pi U\right\|_\infty \\
&\leq
\gamma^2\left\|V-U\right\|_\infty.
\end{aligned}
\label{eq:contraction}
\end{equation}
By the Banach fixed-point theorem, $\mathcal T_{\mathrm{CAST}}$ therefore admits a unique fixed point $V_{\text{CAST}}$ satisfying~\eqref{eq:hybrid_fixed_point}.
\end{proof}

Repeated application of the operator gives
\begin{equation}
V_{CAST}=(\mathcal T_\beta\mathcal T_\pi)^m V_{CAST},\qquad m\geq1,
\label{eq:operator_iteration}
\end{equation}
and expanding the resulting sequence yields
\begin{equation}
V_{CAST}(z)=\mathbb E\left[r_0^\beta+\gamma r_1^\pi+\gamma^2 r_2^\beta+\gamma^3 r_3^\pi+\cdots\right],
\label{eq:alternating_return}
\end{equation}
where the expectation is taken over the stochastic actions sampled from $\beta$ and $\pi_\theta$. Eq.~\eqref{eq:alternating_return} shows that $V_{\text{CAST}}$ discounts rewards under $\beta$ and $\pi_\theta$ in alternation, placing it between $V_\beta$ and $V_{\pi_\theta}$ rather than coinciding with either.

\definecolor{avgtint}{RGB}{236,242,250}
\definecolor{castchange}{RGB}{236,242,250}
\setlength{\fboxsep}{0pt}

\begin{algorithm}[t]
\caption{CAST: one training update. 
\colorbox{castchange}{Marks changes} relative to \cite{hansen2024tdmpc2} (Sec.~\ref{sec:background}).}
\label{alg:vpi}
\begin{algorithmic}[1]
\Input encoder $h_\xi$, dynamics $d_\xi$, reward $R_\xi$, 
\colorbox{castchange}{value ensemble} \colorbox{castchange}{$V_\psi$ \& target $\bar V_\psi$}, 
policy $\pi_\theta$; replay buffer $\mathcal D$, policy optimization horizon $H_\pi$

\Statex \textit{Plan and collect}
\State $z_t \gets h_\xi(s_t)$
\State $a_t \gets \mathrm{MPPI}\big(z_t;\,
\hat G(z_t,a_{t:t+H}) =
\textstyle\sum_{h=0}^{H-1}\gamma^h R_\xi(z_{t+h},a_{t+h})
+ \colorbox{castchange}{$\gamma^H \bar V_\psi(z_{t+H})$}\big)$
\Comment{Eq.~\eqref{eq:vpi_value}}

\State $s_{t+1},r_t \gets env.step(s_t,a_t)$

\State store $(s_t,a_t,r_t,s_{t+1})$ in $\mathcal D$

\Statex
\Statex \textit{Sample and build the hybrid value target}

\State $\{s_t,a_t,r_t,s_{t+1}\}_{t=0}^{H} \sim \mathcal D$

\State $z_t \gets h_\xi(s_t)$, \;
$z_{t+1} \gets h_\xi(s_{t+1})$, \;
\colorbox{castchange}{$\tilde a_{t+1} \gets \pi_\theta(z_{t+1})$}

\State \colorbox{castchange}{$
y_t \gets r_t + \gamma\big[
R_\xi(z_{t+1},\tilde a_{t+1})
+ \gamma\,\bar V_\psi(d_\xi(z_{t+1},\tilde a_{t+1}))
\big]$}

\Statex
\Statex \textit{Update world model and critic}

\State $\mathcal L \gets
\mathcal L_{\text{consist}}
+ \mathcal L_{\text{reward}}
+ \colorbox{castchange}{$\mathcal L_{\text{value}}(V_\psi,y)$}$

\State update $h_\xi,d_\xi,R_\xi,
\colorbox{castchange}{$V_\psi$}$
with $\nabla\mathcal L$
\Comment{single optimizer, $\pi_\theta$ excluded}

\Statex
\Statex \textit{Update policy: 1-step expanded gradient}

\For{$t=0,\dots,H_{\pi}$}
    \State \colorbox{castchange}{$a \gets \pi_\theta(z_t)$}
    \State \colorbox{castchange}{$
    J(z_t;\theta) \gets
    R_\xi(z_t,a)
    + \gamma\,\bar V_\psi(d_\xi(z_t,a))$}
    \Comment{Eq.~\eqref{eq:j1}}
\EndFor

\State $\mathcal L_\pi \gets
-\sum_{t=0}^{H}\rho^t
\big(
\alpha\,\mathcal H(\pi_\theta(\cdot|z_t))
+ \colorbox{castchange}{$J(z_t;\theta)$}
\big)$

\State update $\theta$ with $\nabla_\theta\mathcal L_\pi$

\Statex
\Statex \textit{Update target network}

\State $\bar\psi \gets \tau\psi+(1-\tau)\bar\psi$

\end{algorithmic}
\end{algorithm}

\subsection{Expanded Policy Gradient Through the World Model}
\label{sec:expanded}

Policy improvement follows the same idea as Sec.\ \ref{sec:qtov}, so wherever the base method would query $Q_\psi(z,\pi_\theta(z))$, we substitute the composed estimator $\hat Q$ of~\eqref{eq:composed_q}, evaluated at the policy action:
\begin{equation}
\hat J_1(z;\theta) = R_\xi\big(z,\pi_\theta(z)\big) + \gamma\,\bar V_\psi\big(d_\xi(z,\pi_\theta(z))\big).
\label{eq:j1}
\end{equation}
This is a one-step Bellman backup taken directly through the differentiable reward and dynamics heads. Recursively substituting~\eqref{eq:bellman_q}--\eqref{eq:bellman_v} into themselves $H_\pi$ times, with $a^{(i)} \sim \pi(\cdot\mid z^{(i)})$ and $z^{(i+1)} = d(z^{(i)}, a^{(i)})$ for $i=0,\dots,H_\pi-1$, gives the standard $k$-step Bellman identity for the value function $V_\pi$, for any $H_\pi\geq1$:
\begin{equation}
V_\pi(z) = \mathbb E\!\left[\sum_{i=0}^{H_\pi-1}\gamma^i R\big(z^{(i)},a^{(i)}\big) + \gamma^{H_\pi} V_\pi\big(z^{(H_\pi)}\big)\right].
\label{eq:kstep_identity}
\end{equation}
We use the corresponding learned-model estimator by replacing the true dynamics and reward with $d_\xi$ and $R_\xi$, the policy with $\pi_\theta$, and the terminal value with the frozen target critic $\bar V_\psi$, while evaluating a single sampled trajectory,
\begin{equation}
J_k(z^{0};\theta) \;=\; \sum_{i=0}^{H_\pi-1}\gamma^i\, R_\xi\big(z^{(i)}, a^{(i)}\big) \;+\; \gamma^{H_\pi}\, \bar V_\psi\big(z^{(H_\pi)}\big).
\label{eq:expanded_return}
\end{equation}
Using the reparameterized policy, gradients flow through the full imagined chain $\pi_\theta \to R_\xi \to d_\xi \to \cdots \to \bar V_\psi$. The world model is thus seen as a fixed, differentiable simulator for policy improvement. We treat $H_\pi{=}1$ as our primary configuration throughout, and study $H_\pi{>}1$ as an ablation in Sec.~\ref{sec:experiments}.

\subsection{Training Algorithm}

Algorithm~\ref{alg:vpi} summarizes one update of CAST. Data collection is unchanged, the planner of~\eqref{eq:vpi_value}, seeded by the policy, selects environment actions following an MPC scheme: only the first action is applied to the real environment. The only changes are the critic's input, and the value target and policy objective that follow from it, both derived above.

\begin{figure*}[t!]
\centering
\includegraphics[width=\textwidth]{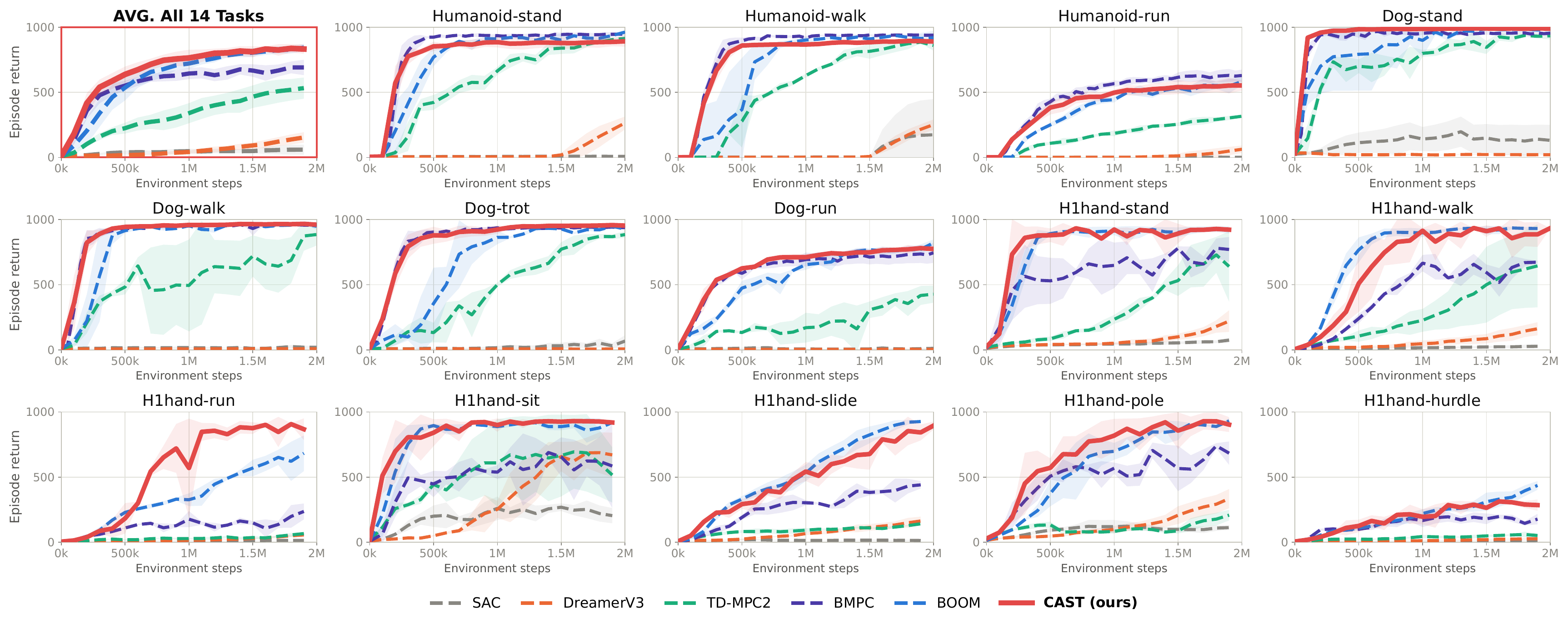}
\caption{Episode return vs.\ environment steps for our method against BOOM \cite{zhan2025boom} and BMPC \cite{wang2025bmpc} (with SAC \cite{haarnoja2018sac}, DreamerV3 \cite{hafner2023dreamerv3}, and TD-MPC2 \cite{hansen2024tdmpc2} also shown), on the 7 DMControl and 7 HumanoidBench tasks of Sec.~\ref{sec:exp_setup}. Top-Left panel: average return across the 14 tasks.}
\label{fig:grid}
\end{figure*}

\begin{figure}[t]
\centering
\includegraphics[width=\linewidth]{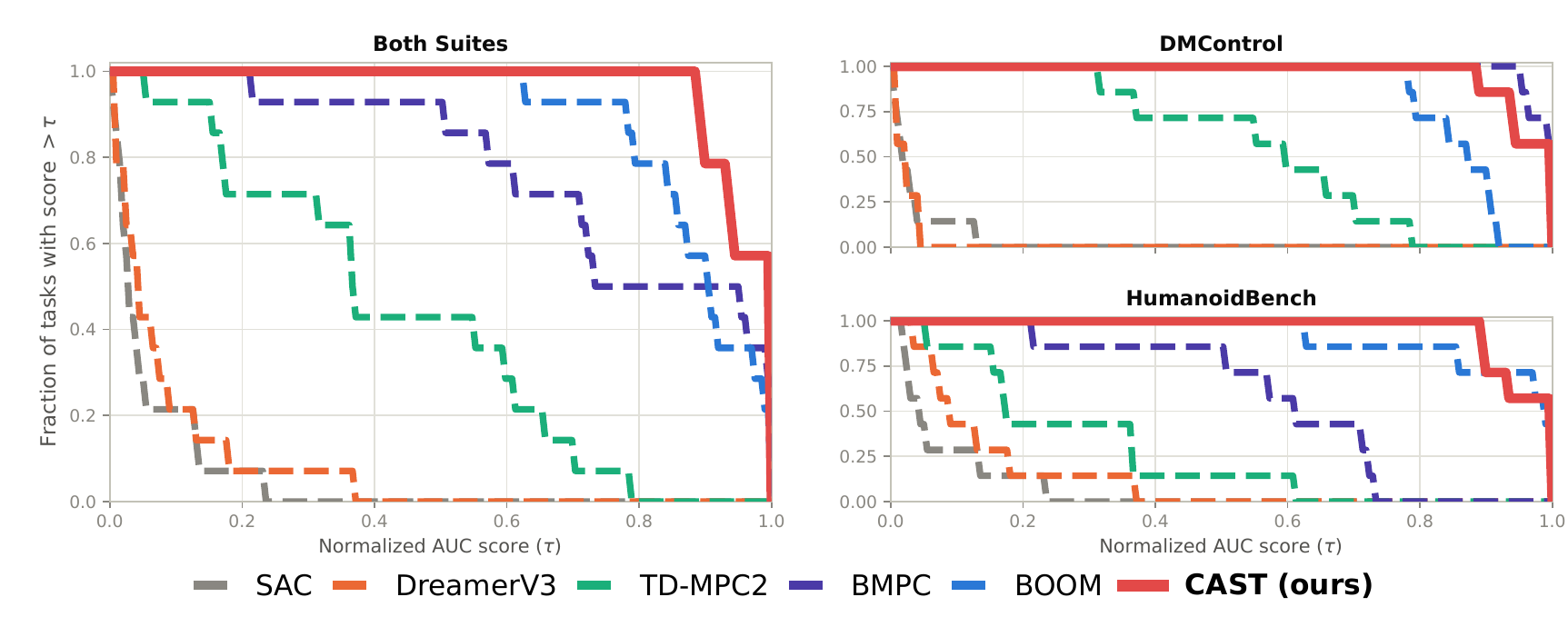}
\caption{\textbf{Performance profile of AUC-to-2M scores}, normalized per task by the best AUC any method reaches on that task: pooled across all 14 tasks (left), and per suite (top right: DMControl, bottom right: HumanoidBench). Each curve shows the fraction of tasks on which a method's normalized score exceeds $\tau$, for all five baselines and our method.}
\label{fig:auc_profile}
\end{figure}

\section{Experiments}
\label{sec:experiments}

We evaluate CAST on a diverse set of high-dimensional continuous-control tasks with the goal of answering three questions.
(i) Does replacing the action-value critic with alternating state-value targets improve sample efficiency?
(ii) Does this improvement preserve strong asymptotic performance?
(iii) What is the contribution of each component of CAST to the overall improvement?

\subsection{Experimental Setup}
\label{sec:exp_setup}

We compare CAST against five representative RL algorithms spanning both model-free and model-based approaches. SAC~\cite{haarnoja2018sac} provides a strong model-free baseline, while DreamerV3~\cite{hafner2023dreamerv3} represents latent imagination-based RL. Our primary comparisons, however, are against the recent family of planning-based methods: TD-MPC2~\cite{hansen2024tdmpc2}, BMPC~\cite{wang2025bmpc}, and BOOM~\cite{zhan2025boom}. These methods share the same latent world model and planning framework, making them the most relevant baselines for evaluating the proposed critic formulation. 

Experiments are conducted on the benchmark introduced by BOOM~\cite{zhan2025boom}, comprising seven DeepMind Control Suite~\cite{tassa2018dmcontrol} tasks and seven HumanoidBench~\cite{sferrazza2024humanoidbench} locomotion tasks. These domains feature large state and action spaces, making both planning and value estimation more challenging than in standard low-dimensional benchmarks.

Unless otherwise stated, CAST employs the alternating state-value target introduced in Section~\ref{sec:qtov} together with a one-step expanded policy gradient ($H_\pi=1$). All methods are evaluated using three random seeds, and shaded regions denote 95\% confidence intervals, computed as $\text{mean} \pm 1.96 \; \text{std-dev} / \sqrt{n}$, where $n$ is the number of seeds. Code will be made publicly available upon acceptance.

\subsection{Sample efficiency}
\label{sec:learning_dynamics}

The primary objective of CAST is to improve sample efficiency by learning a critic closer to the value of $\beta$, without the instability of bootstrapping from $\beta$ alone. We therefore begin by comparing the learning dynamics of all methods throughout training.

Fig.~\ref{fig:grid} reports learning curves for all fourteen benchmark environments, together with the average return across tasks shown in the upper-left panel. CAST matches or exceeds the performance of the selected baselines on the majority of environments and consistently attains high returns earlier during training. These improvements are most evident on several HumanoidBench tasks, including \texttt{H1hand-run}, \texttt{H1hand-pole}, and \texttt{Dog-run}, where CAST reaches strong performance with fewer environment interactions than BMPC and BOOM. Similar trends are observed across the DeepMind Control Suite, indicating that the proposed critic formulation improves learning efficiency across domains with different dynamics and action dimensionalities.

Fig.~\ref{fig:auc_profile} reports a performance profile of AUC scores (Area Under the Curve): for each task, the area under a method's mean return curve from 0 to 2M steps, divided by the number of training steps and normalized by the best AUC any method reaches on that task, then pooled across all 14 tasks and per suite. For a threshold $\tau$ swept from 0 to 1, each curve plots the fraction of tasks on which a method's normalized score exceeds $\tau$. On DMControl, our curve and BMPC's both remain at 1.0 up to $\tau\approx0.9$; BOOM's curve falls off earlier. On HumanoidBench, our curve remains at 1.0 up to $\tau\approx0.85$; BMPC's falls off earliest among the three methods. In average, CAST consistently achieves the best AUC performance profile among all the baselines, confirming that the improvements observed in the learning curves translate into faster policy learning.

\subsection{Final Performance Comparison}
\label{sec:main_results}

Table~\ref{tab:overall_avg_1M_2M_top3_tinted} reports the average episode return after one and two million environment interactions. At 1M interactions, CAST achieves the highest average performance across the benchmark, outperforming both BMPC and BOOM while obtaining the best results on several HumanoidBench environments. These results are consistent with the learning curves presented in the previous section and indicate that the improved learning dynamics of CAST translate into stronger intermediate performance.

After 2M interactions, all planning-based methods approach convergence and the performance gap correspondingly decreases. Nevertheless, CAST remains competitive across the benchmark, outperforming BMPC and achieving performance comparable to BOOM.

\definecolor{checkpointtint}{RGB}{255, 255, 255} 
\definecolor{avgtint}{RGB}{236,242,250}

\begin{table*}[t]
\centering
\caption{
Episode return after 1M and 2M environment interactions
(mean $\pm$95\% CI over seeds).
Best results are shown in \textbf{bold}; second-best are
\underline{underlined}.}
\label{tab:overall_avg_1M_2M_top3_tinted}
\renewcommand{\arraystretch}{0.85}
\resizebox{\textwidth}{!}{
\begin{tabular}{cl cccc|>{\columncolor{checkpointtint}}c>{\columncolor{checkpointtint}}c>{\columncolor{checkpointtint}}c>{\columncolor{checkpointtint}}c}
\toprule
& & \multicolumn{4}{c|}{\textbf{1M Environment steps}}
& \multicolumn{4}{c}{\textbf{2M Environment steps}} \\
\cmidrule(r){3-6}
\cmidrule(l){7-10}
& Task
& TD-MPC2
& BMPC
& BOOM
& \textbf{CAST}
& TD-MPC2
& BMPC
& BOOM
& \textbf{CAST}
\\
\midrule
\multirow{7}{*}{\rotatebox{90}{\textit{DMControl}}}
& Humanoid-stand
& 664 $\pm$ 30
& \textbf{935 $\pm$ 10}
& \underline{910 $\pm$ 19}
& 885 $\pm$ 40
& 913 $\pm$ 17
& \underline{946 $\pm$ 17}
& \textbf{962 $\pm$ 12}
& 891 $\pm$ 41
\\
& Humanoid-walk
& 628 $\pm$ 15
& \textbf{917 $\pm$ 12}
& \underline{901 $\pm$ 22}
& 867 $\pm$ 20
& 860 $\pm$ 40
& \textbf{939 $\pm$ 4}
& \underline{918 $\pm$ 53}
& 890 $\pm$ 13
\\
& Humanoid-run
& 185 $\pm$ 22
& \textbf{566 $\pm$ 24}
& 443 $\pm$ 21
& \underline{497 $\pm$ 23}
& 316 $\pm$ 10
& \textbf{629 $\pm$ 46}
& \underline{583 $\pm$ 30}
& 552 $\pm$ 39
\\
& Dog-stand
& 799 $\pm$ 26
& \underline{953 $\pm$ 11}
& 898 $\pm$ 124
& \textbf{987 $\pm$ 4}
& 933 $\pm$ 28
& 954 $\pm$ 19
& \underline{985 $\pm$ 3}
& \textbf{988 $\pm$ 5}
\\
& Dog-walk
& 494 $\pm$ 352
& 952 $\pm$ 15
& \underline{952 $\pm$ 5}
& \textbf{958 $\pm$ 16}
& 885 $\pm$ 85
& \textbf{962 $\pm$ 11}
& 949 $\pm$ 27
& \underline{961 $\pm$ 3}
\\
& Dog-trot
& 500 $\pm$ 15
& \textbf{936 $\pm$ 10}
& 863 $\pm$ 53
& \underline{923 $\pm$ 32}
& 884 $\pm$ 25
& \underline{939 $\pm$ 13}
& 927 $\pm$ 9
& \textbf{953 $\pm$ 15}
\\
& Dog-run
& 170 $\pm$ 114
& \underline{693 $\pm$ 54}
& 654 $\pm$ 41
& \textbf{711 $\pm$ 25}
& 427 $\pm$ 66
& 745 $\pm$ 78
& \textbf{821 $\pm$ 26}
& \underline{774 $\pm$ 25}
\\
\midrule
\multirow{7}{*}{\rotatebox{90}{\textit{HumanoidBench}}}
& H1hand-stand
& 234 $\pm$ 34
& 648 $\pm$ 175
& \underline{912 $\pm$ 32}
& \textbf{923 $\pm$ 11}
& 640 $\pm$ 270
& 770 $\pm$ 96
& \underline{920 $\pm$ 28}
& \textbf{922 $\pm$ 7}
\\
& H1hand-walk
& 226 $\pm$ 88
& 665 $\pm$ 63
& \underline{898 $\pm$ 33}
& \textbf{914 $\pm$ 22}
& 644 $\pm$ 318
& 673 $\pm$ 12
& \underline{931 $\pm$ 12}
& \textbf{934 $\pm$ 2}
\\
& H1hand-run
& 28 $\pm$ 4
& 177 $\pm$ 62
& \underline{326 $\pm$ 49}
& \textbf{570 $\pm$ 379}
& 66 $\pm$ 9
& 236 $\pm$ 61
& \underline{682 $\pm$ 136}
& \textbf{867 $\pm$ 83}
\\
& H1hand-sit
& 609 $\pm$ 336
& 538 $\pm$ 52
& \underline{888 $\pm$ 16}
& \textbf{899 $\pm$ 51}
& 516 $\pm$ 212
& 585 $\pm$ 322
& \underline{918 $\pm$ 5}
& \textbf{919 $\pm$ 15}
\\
& H1hand-slide
& 92 $\pm$ 17
& 303 $\pm$ 20
& \underline{531 $\pm$ 63}
& \textbf{543 $\pm$ 73}
& 141 $\pm$ 18
& 440 $\pm$ 29
& \textbf{926 $\pm$ 9}
& \underline{895 $\pm$ 31}
\\
& H1hand-pole
& 83 $\pm$ 15
& 568 $\pm$ 43
& \underline{698 $\pm$ 64}
& \textbf{818 $\pm$ 128}
& 208 $\pm$ 40
& 683 $\pm$ 90
& \textbf{930 $\pm$ 21}
& \underline{903 $\pm$ 57}
\\
& H1hand-hurdle
& 44 $\pm$ 13
& 166 $\pm$ 50
& \textbf{220 $\pm$ 33}
& \underline{195 $\pm$ 58}
& 51 $\pm$ 14
& 177 $\pm$ 40
& \textbf{436 $\pm$ 34}
& \underline{296 $\pm$ 66}
\\
\midrule
\multicolumn{2}{l}{\cellcolor{avgtint}\textbf{Average (14 Tasks)}}
& \cellcolor{avgtint}340 $\pm$ 77
& \cellcolor{avgtint}644 $\pm$ 43
& \cellcolor{avgtint}\underline{721 $\pm$ 41}
& \cellcolor{avgtint}\textbf{764 $\pm$ 63}
& \cellcolor{avgtint}535 $\pm$ 82
& \cellcolor{avgtint}691 $\pm$ 60
& \cellcolor{avgtint}\textbf{849 $\pm$ 29}
& \cellcolor{avgtint}\underline{839 $\pm$ 29}
\\
\bottomrule
\end{tabular}
}
\end{table*}

\subsection{Ablation Study}
\label{sec:ablations}

We isolate the contribution of each design choice in Secs.~\ref{sec:qtov}--\ref{sec:expanded}. All ablations use the same architecture, planner, and training budget as the main comparison, varying only the axis under study.

\paragraph{Hybrid value target} We compare our hybrid target against two simpler alternatives: bootstrapping purely off-policy from the behavior policy $\beta$ buffer,
\begin{equation}
y_t^{\text{off}} = r_t + \gamma \bar V_\psi(z_{t+1}),
\end{equation}
and bootstrapping purely on-policy by resampling the action under $\pi_\theta$ at the real state $z_t$, following the target used in BMPC \cite{wang2025bmpc},
\begin{equation}
y_t^{\text{on}} = R_\xi(z_t, \tilde a_t) + \gamma \bar V_\psi\big(d_\xi(z_t, \tilde a_t)\big), \quad \tilde a_t \sim \pi_\theta(\cdot \mid z_t).
\end{equation}
Fig.~\ref{fig:ablation_vtarget} reports this comparison on Dog-run and Humanoid-run, holding $H_\pi=1$ fixed. On both tasks, the purely off-policy variant is consistently the worst of the three throughout training, consistent with the mechanism described in Sec.~\ref{sec:intro}: bootstrapping directly from $V_\beta$ yields abrupt policy updates that destabilize learning. The purely on-policy variant is more competitive but still lags behind.

\paragraph{Behavior-to-policy ratio} We further vary the number of real transitions under $\beta$ preceding the single imagined step under $\pi_\theta$, denoted $H_V \in \{1, 2, 3\}$, with $H_V=1$ recovering~\eqref{eq:y_2step}:
\begin{equation}
\begin{split}
y_t^{(H_V)} = \sum_{i=0}^{H_V-1} \gamma^i r_{t+i} \;+\; \gamma^{H_V}\Big[R_\xi(z_{t+H_V}, \tilde a_{t+H_V}) \\
{}+ \gamma \bar V_\psi(\tilde z_{t+H_V+1})\Big], \quad \tilde a_{t+H_V} \sim \pi_\theta(\cdot \mid z_{t+H_V}).
\end{split}
\end{equation}
Fig.~\ref{fig:ablation_hv} reports this comparison on Dog-run and Humanoid-run. On Dog-run, $H_V=2$ and $H_V=3$ remain competitive with $H_V=1$, while on Humanoid-run both alternatives lag behind throughout training, indicating that anchoring the target to a single real transition before reintroducing $\pi_\theta$ is a well balanced choice.

\begin{figure}[t]
\centering
\includegraphics[width=\linewidth]{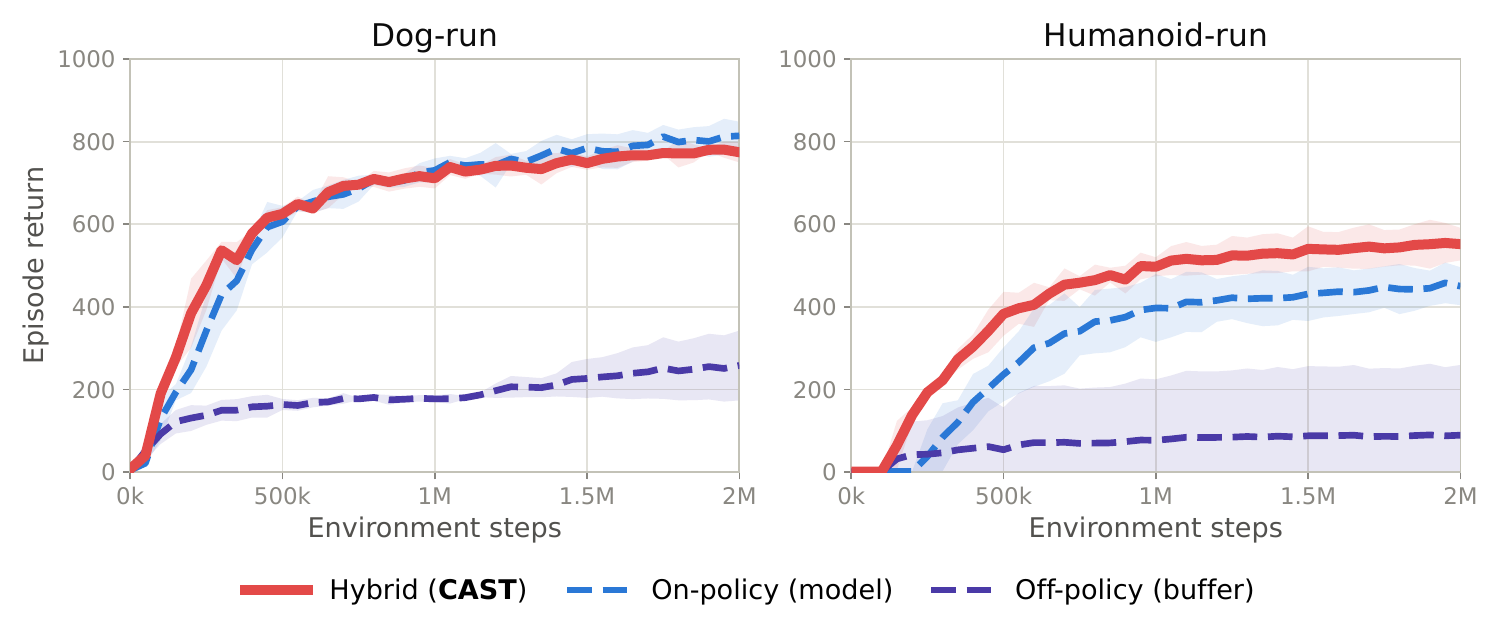}
\caption{\textbf{Value-target ablation} on Dog-run (left) and Humanoid-run (right): our hybrid target vs.\ purely on-policy and purely off-policy alternatives, $H_\pi=1$ throughout. Mean $\pm$ 95\% CI over three seeds.}
\label{fig:ablation_vtarget}
\end{figure}

\begin{figure}[t]
\centering
\includegraphics[width=\linewidth]{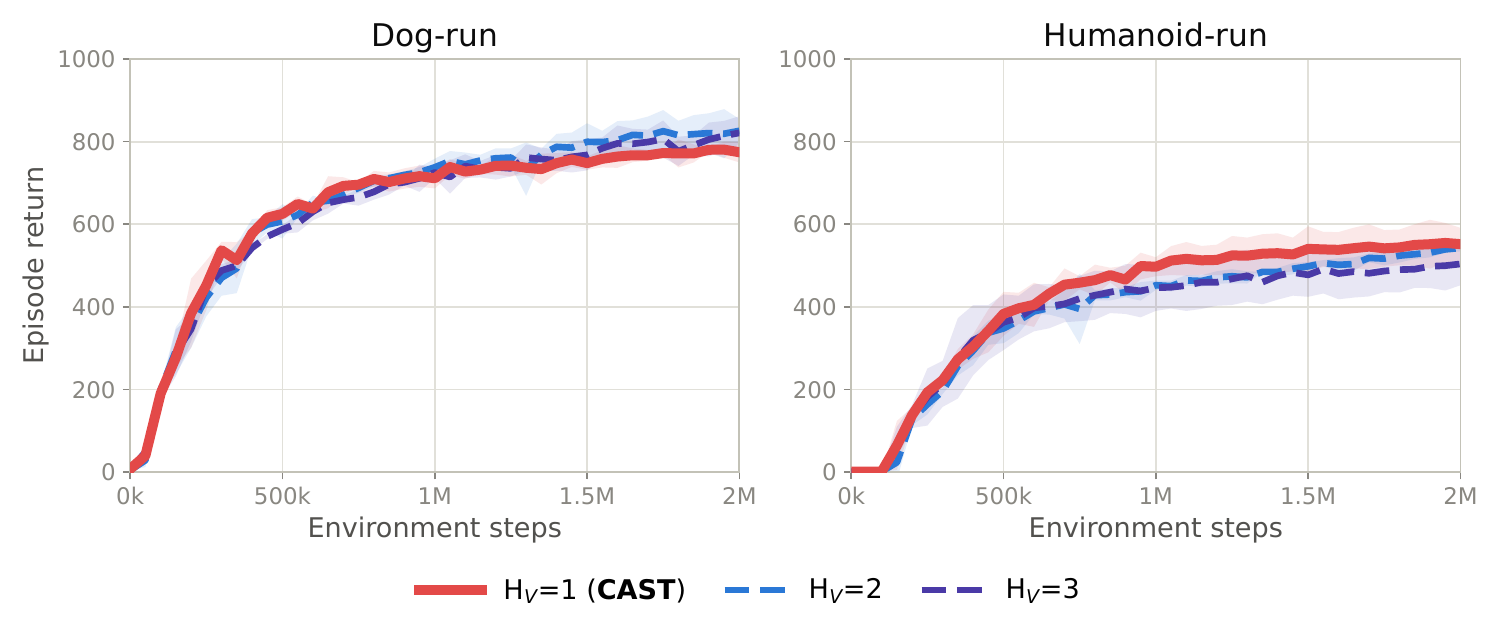}
\caption{\textbf{Behavior-to-policy ratio ablation} on Dog-run (left) and Humanoid-run (right): $H_V=1$ (\textbf{CAST}) vs. $H_V=2$ and $H_V=3$, i.e. anchoring the value target to more real transitions under $\beta$ before the single imagined step under $\pi_\theta$. Mean $\pm$ 95\% CI over three seeds.}
\label{fig:ablation_hv}
\end{figure}

\begin{figure}[t]
\centering
\includegraphics[width=\linewidth]{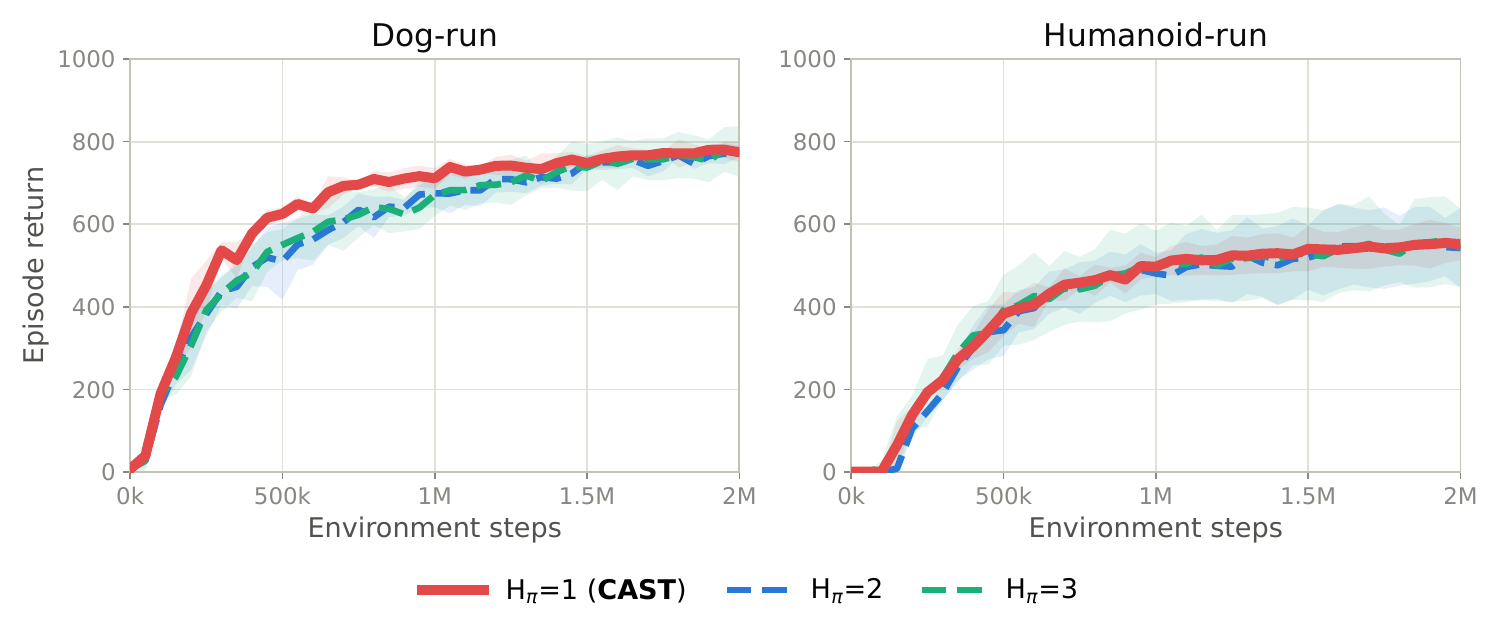}
\caption{\textbf{Expanded-gradient-horizon ablation} on Dog-run (left) and Humanoid-run (right): $H_\pi=1$ (ours) vs.\ $H_\pi=2$ and $H_\pi=3$, hybrid value target fixed throughout. Mean $\pm$ 95\% CI over three seeds.}
\label{fig:ablation_pihorizon}
\end{figure}

\paragraph{Expanded gradient horizon} We vary $k = H_\pi \in \{1,2,3\}$, with $k{=}1$ as the reference configuration used everywhere else in this paper, holding the hybrid value target of Eq.~\eqref{eq:y_2step} fixed. This isolates the effect of unrolling the policy gradient further through the learned model before bootstrapping. Fig.~\ref{fig:ablation_pihorizon} reports this comparison on the same two tasks. On both, $k{=}2$ and $k{=}3$ track each other almost exactly throughout, and track $k{=}1$ closely as well, with $k{=}1$ holding a small lead on both tasks and with smaller CI.

\begin{figure*}[t!]
\centering
\includegraphics[width=\linewidth]{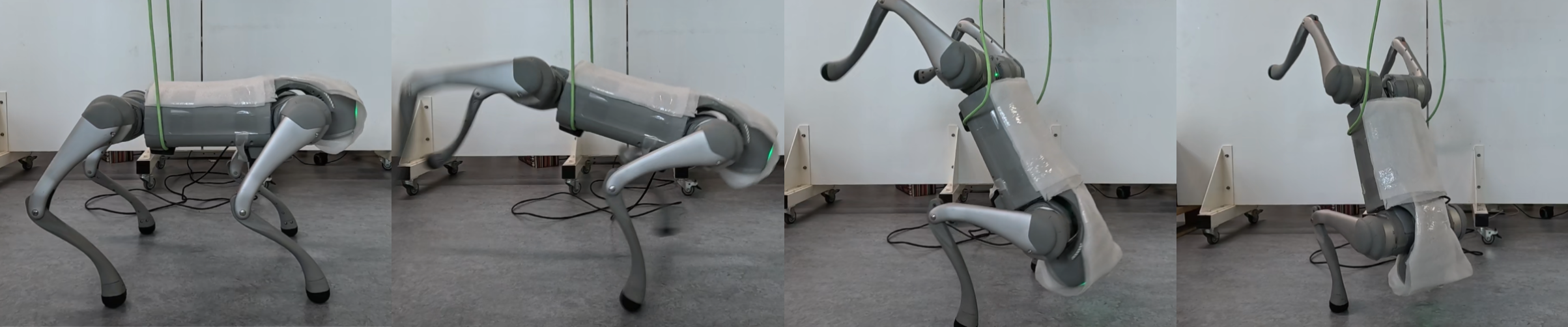}
\caption{\textbf{Hardware deployment.} Snapshots of the learned policy $\pi_\theta$ executing a dynamic handstand on the Unitree Go2. See companion video.}
\label{fig:real_world}
\end{figure*}

\subsection{Real-World Deployment}
\label{sec:deploy-results}

To demonstrate the practical applicability of CAST, we deploy a policy trained entirely in simulation on a Unitree Go2 quadruped. The task consists of performing a dynamic handstand by lifting the rear legs and balancing only on the front legs. Training combines domain randomization over the robot's dynamics with a curriculum that ramps up randomization strength as the policy's stand success rate improves, and the trained policy $\pi_\theta$ is deployed directly on hardware without online planning.

Figure~\ref{fig:real_world} shows the deployment setup and representative snapshots of the learned behavior. The policy successfully transfers from simulation to hardware and executes stable handstand motions in real time. While this experiment is intended as a proof of deployment rather than a comparative evaluation, it demonstrates that CAST produces policies robust enough for direct real-world transfer.

\addtolength{\textheight}{-0.7cm}

\section{Conclusion}

We presented CAST, a planning-based model-based RL framework that replaces the action-value critic with a state-value formulation trained toward the value of the behavior policy $\beta$. Since bootstrapping directly from $\beta$ performs poorly in practice, we introduced a regularized hybrid Bellman target that mixes real transitions under $\beta$ with imagined transitions under the current policy obtained by k-steps rollout with the learned dynamics, with a characterized fixed point. Experiments on high-dimensional continuous-control tasks show strong sample efficiency and competitive final performance against established model-free and model-based baselines. We further demonstrated CAST on a Unitree Go2 quadruped using asynchronous predictive MPC, showing that planning latency can be hidden from the control loop without additional supervision.

\bibliographystyle{ieeetr-mod}
\bibliography{vpi_paper}

\end{document}